\documentclass{article}
\usepackage[preprint]{nips}
\usepackage[utf8]{inputenc}
\usepackage[T1]{fontenc}    
\usepackage{hyperref}       
\usepackage{url}
\usepackage{booktabs}
\usepackage{nicefrac}       
\usepackage{microtype}
\usepackage{amsmath}
\usepackage{amsfonts}
\usepackage{amssymb}
\usepackage{amsthm}
\usepackage{bm}
\usepackage{float}
\usepackage{tikz}
\usepackage{xcolor}
\setcitestyle{numbers,square,comma}

\tikzstyle{edge} = 
[fill, opacity=.2, fill opacity=.2, line cap=round, line join=round, line width=14pt]
\pgfdeclarelayer{background}
\pgfsetlayers{background,main}
\usetikzlibrary{arrows,patterns}
\usetikzlibrary{shapes.geometric}
\usepackage{subcaption}
\newcommand{\myrightarrow}[1]{\mathrel{\raisebox{6pt}{$\xrightarrow{#1}$}}}

\newtheorem{tm}{Theorem}[section]
\newtheorem{df}[tm]{Definition}
\newtheorem{lm}[tm]{Lemma} 
\newtheorem{co}[tm]{Corollary}

\theoremstyle{definition}
\newtheorem{ex}[tm]{Example}
\newtheorem{re}[tm]{Remark}

\newcommand{\X}{\mathbf X}
\newcommand{\Y}{\mathbf Y}

\newcommand{\E}{\mathbf E}

\renewcommand{\c}{\mathbf c}
\newcommand{\x}{\mathbf x}
\newcommand{\y}{\mathbf y}

\renewcommand{\P}{\mathbf {Pr}}
\renewcommand{\S}{\mathbf S}
\newcommand{\R}{\mathbb R}
\newcommand{\A}{\mathcal A}

\title{McDiarmid-Type Inequalities for Graph-Dependent Variables and Stability Bounds}

\author{
\textbf{Rui (Ray) Zhang}
\thanks{
This work was done when this author was a master student at 
the Institute of Computing Technology, Chinese Academy of Sciences
and University of Chinese Academy of Sciences.
This research forms part of Rui (Ray) Zhang's master thesis 
submitted to the University of Chinese Academy of Sciences in May 2019.
} \\
School of Mathematics \\
Monash University \\
\texttt{rui.zhang@monash.edu} \\
\And
\textbf{Xingwu Liu}
\thanks{
Corresponding author
} \\
Institute of Computing Technology, \\
Chinese Academy of Sciences. \\
University of Chinese Academy of Sciences \\
\texttt{liuxingwu@ict.ac.cn} \\
\And
\textbf{Yuyi Wang} \\
ETH Zurich, Switzerland \\
X-Order Lab, China\\
\texttt{yuyiwang920@gmail.com}\\
\And
\textbf{Liwei Wang} \\
Key Laboratory of Machine Perception, MOE, \\
School of EECS, Peking University \\
Center for Data Science, Peking University \\
\texttt{wanglw@cis.pku.edu.cn} \\
}

\begin{document}
\maketitle

\begin{abstract}

%A crucial assumption in most statistical learning theory is that samples are independently and identically distributed (i.i.d.). 
%However, for many real applications the i.i.d.\ assumption does not hold. 
%We consider learning problems in which examples are dependent and their dependency relation is characterized by a graph. 
%To establish generalization theory for learning with dependent data, 
%we first prove novel McDiarmid-type concentration inequalities for Lipschitz functions of graph-dependent random variables. 
%Based on these new inequalities we build stability bounds for learning from graph-dependent data.

A crucial assumption in most statistical learning theory is that samples are independently and identically distributed (i.i.d.). However, for many real applications, the i.i.d. assumption does not hold. We consider learning problems in which examples are dependent and their dependency relation is characterized by a graph. To establish algorithm-dependent generalization theory for learning with non-i.i.d. data, we first prove novel McDiarmid-type concentration inequalities for Lipschitz functions of graph-dependent random variables. We show that concentration relies on the forest complexity of the graph, which characterizes the strength of the dependency. We demonstrate that for many types of dependent data, the forest complexity is small and thus implies good concentration. Based on our new inequalities we are able to build stability bounds for learning from graph-dependent data.

\end{abstract}

\section{Introduction}
Generalization theory is at the foundation of machine learning. It quantifies how accurate a model would predict on the test data which the learning algorithm is not able to access during training. It usually relies on a crucial assumption: The data are independently and identically distributed (i.i.d.). The i.i.d.\ assumption allows one to use many powerful tools from probability to prove strong generalization error bounds. However, in real applications, the data are often non-i.i.d.\ i.e., the data collected can be dependent. There have been extensive discussions on why and how the data are dependent. We refer the readers to \cite{dehling2002empirical,amini2015learning}.

Establishing generalization theory for dependent data has received a lot of attention \cite{mohri2008stability,mohri2009rademacher,ralaivola2010chromatic,kuznetsov2017generalization,yi2018learning}. A major line of research in this direction models the data dependency by various types of mixing such as $\alpha$-mixing~\cite{rosenblatt1956central},
$\beta$-mixing~\cite{volkonskii1959some}, $\phi$-mixing~\cite{ibragimov1962some}, $\eta$-mixing~\cite{kontorovich2007measure}, etc.
Mixing models have been used in statistical learning theory to 
establish generalization error bounds
based on Rademacher complexity~\cite{mohri2009rademacher,kuznetsov2017generalization,mohri2010stability} or algorithmic stability~\cite{mohri2008stability,mohri2010stability,he2016stability}
via concentration results~\cite{kontorovich2008concentration} or independent blocking technique~\cite{yu1994rates}.
In these models, the mixing coefficients measure the extent to which the data are dependent to each other.
Similar to the mixing models, 
learning under Dobrushin's condition \cite{pmlr-v99-dagan19a}
is also investigated via concentration results \cite{kulske2003concentration,chatterjee2005concentration,kontorovich2017concentration} using
Dobrushin's interaction matrix \cite{dobruschin1968description}.
Although the results under the various mixing conditions and Dobrushin's condition are fruitful, they are faced with difficulties in application: It is sometimes difficult to determine the quantitative dependency among data points. On the other hand, determining whether two data are dependent or not is often much easier. In this paper, we focus on such qualitative dependency of data. We use simple graphs as a natural tool to describe the dependency among data, and establish generalization theory for such graph-dependent data. 

A basic building block of generalization theory is concentration inequality. Different settings and different assumptions require different concentration tools. The less we assume, the more powerful tools we need. In order to establish generalization theory for dependent data, standard concentration for i.i.d.\ data no longer applies. One must develop concentration inequalities for dependent data, which is a very challenging task.

In his seminal work~\cite{janson2004large}, Janson proved an elegant concentration inequality for graph-dependent data. The inequality is a beautiful extension of Hoeffding inequality. It bounds the probability that the summation of graph-dependent random variables deviates from its expected value, in terms of the fractional coloring number of the dependency graph. 
%  can be built based on Janson’s concentration inequality. 
Janson's inquality has been extended to any functions that can be decomposed into the summation of some functions of independent random variables~\cite{usunier2006generalization}. 
This extension enables to establish generalization error bounds for graph-dependent data via fractional Rademacher complexity. 

In \cite{ralaivola2010chromatic}, PAC-Bayes bounds for classification with non-i.i.d.\ data are obtained based on fractional colorings of graphs. The results also hold for specific learning settings such as ranking and learning from stationary $\beta$-mixing distributions. In \cite{ralaivola2015entropy}, Ralaivola and Amini established new concentration inequalities for fractionally sub-additive and fractionally self-bounding functions of dependent variables. Their results are based on the fractional chromatic numbes and the entropy method. In \cite{DBLP:conf/alt/WangGR17}, Wang et al.\ used hypergraphs to model dependent random variables that are generated by independent ones. Leveraging the notion of fractional matching, they also establish concentration inequalities of Hoeffding- or Bernstein-type. 

Though fundamental and elegant, the above generalization bounds are algorithm-independent. They considered the complexity of the hypothesis space and data distribution, but does not involve the learning algorithm. 
To derive better generalization bounds, there are growing interests in developing algorithm-dependent generalization theories. This line of research heavily relies on the algorithmic stability. A key advantage of stability bounds is that they are tailored to specific learning algorithms, exploiting their particular properties. 

% They do not depend on complexity measures such as the VC-dimension, covering numbers, or Rademacher complexity, which characterize a class of hypotheses, independently of any algorithm.
% For example, one can prove a generalization error bound for graph-dependent data for any Empirical Risk Minimization (ERM) learner on a hypothesis space with a finite covering number, which is usually called a VC type bound. 

% Despite of a fundamental work in machine learning, VC theory is an algorithm-independent generalization theory. It only bounds the generalization error of the ERM or SRM (structural risk minimization \cite{}) type estimators; and does not consider which learning algorithm is used. 

% The great success of deep learning in the past a few years raises serious questions to classic learning theory \cite{} [LW: cite chiyuan zhang et al ICLR’17 best paper]. The VC dimension of practically used neural networks is often much larger than the number of training data, making the VC type generalization error bound vacuous \cite{}. To better understand deep learning, there are growing interests in developing algorithm-dependent generalization theories. Hardt et al.\ \cite{} analyzed Stochastic Gradient Descent (SGD) algorithm, and proved that it has good generalization. \cite{COLT’17 & my paper} proved Stochastic Langevin Dynamics generalizes well even if one uses a highly overparametrized neural network. These results heavily rely on the algorithmic stability, which is an important algorithm-dependent learning theory. 

How can we establish algorithmic stability theory for graph-dependent data? Note that under the assumption of i.i.d.\ data, Hoeffding-type concentration inequality, which bounds the deviation of sample average from expectation, is not strong enough to prove stability-based generalization. On the contrary, McDiarmid’s inequality characterizes the concentration of general Lipschitz functions of i.i.d.\ random variables, hence serving as the key tool for proving the stability theory. Therefore, to build algorithmic stability theory for non-i.i.d. samples, one has to develop McDiarmid-type concentration for graph-dependent random variables.

In this paper, we prove the first McDiarmid-type concentration inequality for graph-dependent random variables in terms of a new notion called forest complexity, which measures the strength of the dependency. It turns out that for various dependency graphs, it is easy to estimate the forest complexity. The proposed concentration inequality enables us to prove stability-based generalization bounds for graph-dependent data. Our results provide basic tools for understanding learning with overparameterized models.

% In \cite{agarwal2009generalization}, Agarwal and Niyogi proposed an interesting method to deal with ranking problem via stability. Basically, they considered a strong version of uniform stability for ranking algorithm, which is measured in terms of pairs rather than individual examples. In this way, one can avoid handling dependence among data and hence apply standard McDiarmid Inequality.

% There are also many concentration inequalities in mixing models
% ~\cite{marton1998measure,samson2000concentration,chazottes2007concentration,kontorovich2008concentration,paulin2015concentration,kontorovich2017concentration,kulske2003concentration,kontorovich2017concentration,chatterjee2005concentration}.
% They have been used in statistical learning theory to establish generalization error bounds via Rademacher complexity~\cite{mohri2009rademacher,kuznetsov2017generalization,mohri2010stability} or algorithmic stability~\cite{mohri2008stability,mohri2010stability,he2016stability}, by applying the independent blocking technique \cite{yu1994rates} which was originated from Bernstein \cite{bernstein1927extension}. We don't dwell on it since our focus is dependency graphs.

The rest of the paper is organized as follows. 
In section 2, we briefly introduce the notations and related results.
In section 3, we establish McDiarmid-type inequalities for acyclic dependency graphs,
and extend the concentration results to the general dependency graphs.
In section 4, we apply our concentration results to the learning theory 
and establish generalization error bounds for learning graph-dependent data
via algorithmic stability, we also provide an application of learning $m$-dependent data.
Section 5 concludes the paper and points out the future research directions.

%\newpage
\section{Preliminaries}
In this section, we present the notations and the basic McDiarmid's inequality for i.i.d. random variables.

Throughout this paper, let $n$ be a positive integer with $[n]$ standing for the set $\{1,2,\ldots,n\}$. Let $\Omega_i$ be a Polish space for any $i\in [n]$, $\mathbf{\Omega}=\prod_{i\in [n]}\Omega_i$ be the product space, $\R$ be the set of real numbers, $\R_+$ be the set of non-negative real numbers, $\mathbb N_+$ be the set of non-negative integers.

Concentration inequalities are fundamental tools in statistical learning theory.
They are essentially tail probability bounds indicating how much a function of random variables deviates from some value that is usually the expectation. 
Among the most powerful ones is the McDiarmid's inequality which establishes a sharp, even tight in some cases, bound on the concentration, when the function satisfies $\c$-Lipschitz condition (bounded differences condition), namely, does not depend too much on any individual variable. 

\begin{df}[$\c$-Lipschitz]
Given a vector $\c=(c_1,\ldots,c_n) \in \R_+^n$, a function $f:\mathbf{\Omega}\rightarrow \mathbb{R}$ is said to be $\c$-Lipschitz if for any $\mathbf{x}=(x_1,\ldots,x_n),\mathbf{x}'=(x'_1,\ldots,x'_n)\in \mathbf{\Omega}$, it satisfies
\[
| f(\x ) - f(\x ') | \le \sum_{i = 1}^n c_i \mathbf 1_{ \{ x_i \ne x_i' \} },
\]
where $c_i$ is called the $i$-th Lipschitz coefficient of $f$.
\end{df}
\begin{tm}[McDiarmid's inequality~\cite{mcdiarmid1989method}]
Suppose $f: \mathbf{\Omega}\rightarrow \R$ is $\c$-Lipschitz, and $\X=(X_1,\ldots,X_n)$ is a vector of independent random variables with each $X_i$ taking values in $\Omega_i$. Then for any $t>0$, the tail probability satisfies
\begin{equation}\label{eqn:McDiarmid's ie}
\P\left( f(\X) - \E   [f(\X)] \ge t \right)
\le \exp \left(  - \dfrac{2t^2}{ \|\c\|^2_2}  \right). 
\end{equation}
\label{McDiarmid's inequality}
\end{tm}

Notice that the McDiarmid's inequality works for independent random variables.
Janson's Hoeffding-type inequality~\cite{janson2004large} for graph-dependent random variables
is a special case of McDiarmid-type inequality when the function is a summation. 
Specifically, when $f(\X) = \sum_{i=1}^n X_i$ with each $X_i$ ranging over an interval of length $c_i$,
\begin{equation}\label{eqn:janson}
\P\left( \sum_{i=1}^n X_i - \E  \left[\sum_{i=1}^n X_i\right] \ge t \right)\le \exp \left(  - \dfrac{2t^2}{ \chi^*(G) \|\c\|^2_2}  \right),
\end{equation}
where $\c=(c_1,\ldots,c_n)$ and $\chi^*(G)$ is the fractional coloring number 
of a dependency graph $G$ of random variables $\X$.

\section{McDiarmid Concentration for Graph-dependent Random Variables}

In this section we present our first set of main results, the McDiarmid-type concentration inequalities (i.e., concentration of Lipschitz functions) for graph-dependent random variables. The results in this section will serve as the tools for developing learning theory for dependent data.

We start from the simplest case that the dependency graph is acyclic, i.e., trees or forests. We prove McDiarmid-type concentration bounds for trees and forests with very simple forms. These inequalities are then extended to general graphs. To this end, we introduce the notion of forest complexity, which characterizes to what extent a general graph can be best approximated by a forest. We prove McDiarmid-type concentration inequality for general graph-dependent random variables in terms of the forest complexity. Finally we demonstrate that for many important classes of graphs, forest complexity is easy to estimate.

Below we first define the notion of dependency graphs, which is a widely used model in probability, statistics, and combinatorics, see \cite{erdos1975problems,janson1988exponential,chen1978two,baldi1989normal,janson2011random} for examples.
\begin{df}[Dependency Graphs]
An undirected graph $G$ is called a dependency graph of a random vector $\X=(X_1,\ldots,X_n)$ if 
\begin{enumerate}
\item $V(G)=[n]$
\item if $I, J \subset [n]$ are non-adjacent in $G$, then 
$\{X_i\}_{i \in I}$ and $\{X_j\}_{j \in J}$ are independent.
\end{enumerate}
\end{df}

% [LW: need to modify] Figure 1 illustrates the dependency graph of the instance described in Example \ref{wanderingpatients}. 

\subsection{McDiarmid Concentration for Acyclic Graph-dependent Variables}
Our first result is for the case that the dependency graph is a tree. 
%First of all, consider the special case with trees as dependency graphs. 
\begin{tm} 
Suppose that $f: \mathbf{\Omega}\rightarrow \R$ is a $\c$-Lipschitz function and $G$ is a dependency graph of a random vector $\X$ that takes values in $\mathbf{\Omega}$. If $G$ is a tree, then for any $t > 0$, 
the following inequality holds:
\begin{equation}\label{eq:Dependency forest}
\P( f(\X) - \E [f(\X)] \ge t ) 
\le \exp \left( - \dfrac{2t^2}{ 
\sum_{\langle i,j \rangle \in E(G)} ( c_i + c_j)^2 + c_{\min}^2 } \right),
\end{equation}
where $c_{\min}$ is the minimum entry in $\c$.
\label{Concentration Inequality for Dependency Tree}
\end{tm}

The proof of this theorem
relies on decomposing $f(\X) - \E [f(\X)]$ into the summation $\sum_{i = 1}^n V_i$ with $V_i := \E [ f(\X)| X_1, \ldots X_i ] - \E [ f(\X)| X_1, \ldots X_{i-1} ]$. 
We show that each $V_i$ ranges in an interval of length at most $c_i+c_j$, where $j$ is the parent of $i$ in the tree (in the proof, we make the tree rooted by choosing the vertex with the minimum Lipschitz coefficient as the root).
The theorem is then proved by applying the Chernoff-Cram\'{e}r technique to $\sum_{i = 1}^n V_i$. 
For details, please refer to Subsection \ref{tree} in the supplementary materials.

Like McDiarmid's inequality, Theorem \ref{Concentration Inequality for Dependency Tree} also claims a deviation probability bound that decays exponentially. The decay rate is determined by two interplaying factors. One is the Lipschitz coefficient that is inherent to the function. The other is the pattern of the dependency, namely, which random variables are dependent and connected by an edge. 

% Intuitively, in order to guarantee concentration, we should avoid dependence between important variables (namely, variables whose corresponding Lipschitz coefficient is big).

We then generalize the above result to the case where dependency graph $G$ is a forest.
% Roughly speaking, since trees are disconnected with one another in a forest, 
% the results for individual trees altogether lead to the overall bound.
% The formal proof is given in Subsection \ref{forest} of the supplementary materials.

\begin{tm}
Suppose that $f: \mathbf{\Omega}\rightarrow \R$ is a $\c$-Lipschitz function and $G$ is a dependency graph of a random vector $\X$ that takes values in $\mathbf{\Omega}$. If $G$ is a forest consisting of trees $\{ T_i \}_{i \in [k]}$, then for any $t > 0$, 
the following inequality holds:
\begin{equation}\label{eqn:Dependency forest}
\P( f(\X) - \E[f(\X)] \ge t )
\le \exp \left( - \dfrac{2t^2}{ 
\sum_{\langle i,j \rangle \in E(G)} ( c_i + c_j)^2 + \sum^k_{i=1} c_{\min,i}^2 } \right),
\end{equation}
where $c_{\min,i}=\min\{c_j: j\in V(T_i)\}$.
\label{Concentration Inequality for Dependency forest}
\end{tm}
Theorem~\ref{Concentration Inequality for Dependency forest} can be proved in a similar way as Theorem~\ref{Concentration Inequality for Dependency Tree}. The detailed proof is presented in Subsection \ref{forest} of the supplementary materials. 
%For intuition, the bound of the forest is a \emph{combination} of those of the component trees. The combination is reasonable 
%since the trees are disconnected with one another in a forest, the results for individual trees altogether lead to the overall bound. 
% The formal proof is given in Subsection \ref{forest} of the supplementary materials.

We point out that Theorem~\ref{Concentration Inequality for Dependency forest} is a strict generalization of the McDiarmid's inequality for i.i.d. random variables. If all the random variables are independent, i.e., there is no edge in the dependency graph, then it is clear that Eq.\ \eqref{eqn:Dependency forest} degenerates exactly to Eq.\ (\ref{eqn:McDiarmid's ie}). 

Theorem~\ref{Concentration Inequality for Dependency forest} also clearly demonstrates how dependency between random variables affects concentration. The decay rate of the probability that $f(\X)$ deviates from its expectation is approximately reversely proportional to the number of edges in the dependency graph.

%\begin{re}
%When the random variables are mutually independent, $E(G) = \emptyset$. 
%In this case,  Theorem~\ref{Concentration Inequality for Dependency forest} degenerates to 
%the original McDiarmid's inequality (Theorem~\ref{McDiarmid's inequality}).
%\label{degenerate to mcdi}
%\end{re}

\subsection{McDiarmid Concentration for General Graphs}
%In this subsection, we handle random variables with general dependency graphs. 
%The basic idea is to partition the variables into groups so that the dependency graph of these groups is a forest. 
%Then, regarding every group as a new monolithic random variable, we obtain a concentration inequality by applying Theorem \ref{Concentration Inequality for Dependency forest}. Formally, this process is approximating a general dependency graph by an acyclic one, namely, a forest.

In this subsection, we consider general graphs. 
% First note that if the graph is a clique, then there is no guarantee of concentration at all, since there is a case that all the random variables are the same one. Concentration for graph-dependent random variables relies on the sparse structure of the graph. 
Our basic idea for handling general graphs is to use a forest to approximate the graph. Specifically, we partition the variables into groups so that the dependency graph of these groups is a forest. We try to find the optimal forest approximation, which leads to the notion of forest complexity. We then prove McDiarmid-type concentration inequality for general graph-dependent random variables in terms of its forest complexity, which yields a very simple form.

We first define the concept of forest approximation.

\begin{df}[Forest Approximation]
Given a graph $G$, a forest $F$, and a mapping $\phi: V(G)\rightarrow V(F)$, if $\phi(u)=\phi(v)$ or $\langle \phi(u),\phi(v) \rangle\in E(F)$ for any $\langle u,v \rangle\in E(G)$, we say that $(\phi,F)$ is a forest approximation of $G$. Let $\Phi(G)$ denote the set of forest approximations of $G$.
\end{df}

Intuitively, a forest approximation is transforming a graph into a forest by merging vertices and removing the incurred self-loops and multi-edges. In this way, we rule out the redundant variables that heavily depend on others and thus contribute little to concentration.

Based on forest approximation, we define the notion of forest complexity of a graph, which intuitively measures how much the graph looks like a forest. 

\begin{df}[Forest Complexity]\label{lambda}
Given a graph $G$ and any forest approximation $(\phi,F)\in \Phi(G)$ with $F$ consisting of trees $\{ T_i \}_{i \in [k]}$, let 
\[
\lambda_{(\phi,F)} = \sum_{\langle u,v \rangle \in E(F)} 
\left( |\phi^{-1}(u)| + |\phi^{-1}(v)| \right)^2 + \sum_{i=1}^k\min_{u\in V(T_i)}|\phi^{-1}(u)|^2.
\]
We call
\[
\Lambda(G)=\min_{(\phi,F)\in\Phi(G)} \lambda_{(\phi,F)}
\]
the forest complexity of the graph $G$.
\end{df}
%Let 
%\begin{equation}
%\Lambda(G)=\min_{(\phi,G,F)\in\Phi(G)} \lambda_{(\phi,G,F)}.
%\end{equation}\label{lambda}

%Then Lemma~\ref{Concentration Inequality for Dependency Graph} immediately implies

Now we are ready to state our McDiarmid-type concentration inequality for general graph-dependent random variables.

\begin{tm}
Suppose that $f: \mathbf{\Omega}\rightarrow \R$ is a $\c$-Lipschitz function and $G$ is a dependency graph of a random vector $\X$ that takes values in $\mathbf{\Omega}$.  For any $t>0$,  the following inequality holds:
\[
\P( f(\X) - \E   [f(\X)] \ge t )
\le \exp \left( - \dfrac{2t^2}{ \Lambda(G)\| \c \|_{\infty}^2
} \right).
\]
\label{Concentration Inequality for Dependency Forest 1}
\end{tm}

With the tool of forest approximation, we reduce the concentration problem defined on graphs to that defined on forests. Basically, we use a new variable to represent each set of the original variables that are merged together by the forest approximation. The function can be equivalently transformed into a function of the new variables whose dependency graph is the forest. The proof is done by applying Theorem \ref{Concentration Inequality for Dependency forest} to the new function. For details, please refer to Subsection \ref{generalgraph} in the supplementary materials.

Like the above theorems, Theorem \ref{Concentration Inequality for Dependency Forest 1} also establishes an exponentially decaying probability of deviation. The decay rate is totally determined by the Lipschitz coefficient of the function and the forest complexity of the variables' dependency graph. Intuitively, the more the dependency graph looks like a forest, the faster the deviation probability decays. This uncovers how the dependencies among random variables influence concentration.

% \begin{re}
% Consider the special case where the entries of $\c$ do not vary too much. Then, $n\| \c \|_{\infty}^2 \approx \|\c\|^2_2$. 
% In addition, like the fractional coloring number $\chi^*(G)$, the forest complexity $\Lambda(G)$ is also a quantity indicating the dependence among the random variables. 
% Hence, the bound in Theorem~\ref{Concentration Inequality for Dependency Forest 1} is of the same style of that in Janson's inequality (see formula (\ref{eqn:janson}) in the present paper).
% \end{re}

\subsection{Illustrations and Examples}
\label{Examples and Applications}

This subsection consists of two parts. In the first part we review a widely-studied random process that generates dependent data whose dependency graph can be naturally constructed.
% illustrate that some theoretical frameworks widely used in literature to model the generating process of dependent data can naturally lead to dependency graphs. 
In the second part, we deal with some dependency graphs to show that in many cases, the forest complexity is small and easy to estimate.

%The Poisson point process are used to model random events, such as the arrival of customers at a store, phone calls at an exchange or occurrence of earthquakes, distributed in time. 
%In spatial setting, the point process, also known as a spatial Poisson process, 
%are widely used models for count and point data in machine learning field~\cite{linderman2014discovering,kirichenko2015optimality}.

%We describe a natural process that generates various sets of dependent random variables. Calculation indicates that in many cases, it is easy to derive a non-trivial (of order o($n^2$)), small enough upper bound of the forest complexity. We also observe that the forest complexity is a good indicator of the strength of the dependency, in the sense that the sparser the graph is, the smaller the forest complexity is. This justifies establishing concentration inequalities in terms of forest complexity. 
% several applications of our concentration results
% and examples that exhibit \textit{weak dependence}, i.e., $\Lambda(G)=o(n^2)$.
% This will lead to generalization error bound analysis in Remark~\ref{gen bound analysis}.
% Notice that $\Lambda(G)$ is always smaller than $n^2$ (Example~\ref{trivialapprox}) and greater than $n$ (Remark~\ref{degenerate to mcdi}).

Consider a data generating procedure modeled by the \textit{spatial Poisson point process}, 
which is a Poisson point process on $\R^2$  (See~\cite{linderman2014discovering,kirichenko2015optimality} for discussions of using this process to model data collection in various machine learning applications.)
The number of points in each finite region follows a Poisson distribution, 
and the number of points in disjoint regions are independent.
Given a finite set $\mathcal I = \{ I_i \}_{i=1}^n$ of regions in $\R^2$, let $X_i$ be the number of points in region $I_i$, $1\le i\le n$. 
Then the graph $G\left( [n], \{ \langle i, j \rangle: I_i \cap I_j \ne \emptyset \} \right)$ is a dependency graph of the random variables $\{ X_i \}_{i=1}^n$.
% we connect $X_i$ and $X_j$ by an edge $\langle X_i, X_j \rangle$ 
% if $I_i \cap I_j \ne \emptyset$.
% The constructed graph 
% is a dependency graph for $\{ X_i \}_{i=1}^n$.
% Thus, the concentration result in Theorem~\ref{Concentration Inequality for Dependency Forest 1}
% applies.
%\end{ap}

%The spatial Poisson point process has been frequently studied in the random geometric graph theory. See~\cite{penrose2003random} for details.

%Now consider a simple case of the regions where for any $i\ne j, $ $I_i \cap I_j \ne \emptyset$ if and only if $|i-j|= 1$. The dependency graph is a path of length $n$. Likewise, one can easily obtain sets of random variables with various dependency graphs. For example, trees, cycles, grids, just mention a few. We will estimate the forest complexity of each of these dependency graphs.

%Throughout this section, we fix a $\c$-Lipschitz function $f: \mathbf{\Omega}\rightarrow \mathbb R$ and a random vector $\X$ that takes values in $\mathbf{\Omega}$. Assume that $G$ is the dependency graph of $\X$. 
% In order to use Theorem~\ref{Concentration Inequality for Dependency Forest 1},
% the task is reduced to finding an upper bound of $\Lambda(G)$ in every case. 

We present three examples to demonstrate that estimating the forest complexity $\Lambda(G)$ is usually easy. All the examples can naturally appear in the above process.

\begin{ex}[$G$ is a tree]
In this case, the identity map between $G$ and itself is a forest approximation of $G$. Then 
$\Lambda(G) \le |E(G)|(1+1)^2 + 1 = 4n - 3 = O(n)$. We get an upper bound of $\Lambda(G)$ that is linear in the number of variables, which is almost tight compared with Hoeffding's inequality or Janson's result (see (\ref{eqn:janson}) with $\chi^*(G)=2$).
\label{example tree}
\end{ex}

\begin{ex}[$G$ is a cycle $C_n$]
If $n$ is even, a forest approximation is illustrated in Figure~\ref{evencycle}, where the cycle is approximated by a path $F$ of length $\frac{n}{2}$. The approximation $\phi$ maps any vertex of $G$ to the vertex of $F$ having the same shape, so each gray belt stands for a preimage set of $\phi$. We will keep this convention in the rest of this section. By the illustrated forest approximation, 
$\Lambda(G) \le 2 \times (1+2)^2 + (\frac{n}{2}-2)(2+2)^2+1 = 8n-13  = O(n) $.
When $n$ is odd, according to the forest approximation shown in Figure~\ref{oddcycle},
$\Lambda(G)\le (1+2)^2 + (\frac{n - 1}{2}-1)(2+2)^2 +1= 8n - 14 = O(n)$. Since $\chi^*(G)$ is 2 or 3, our bound is again very tight compared with Janson’s result.
\label{example cycle}
\end{ex}

\begin{figure}[H]
\begin{minipage}{.5\textwidth}
\centering
\begin{subfigure}[c]{.4\textwidth}\centering
\begin{tikzpicture}
\node[diamond,draw=black,fill=white,scale=.8] (i) at (0:1cm) {$$};
\node[circle,draw=black,fill=white] (j) at (60:1cm) {$$};
\node[diamond,draw=black,fill=white,scale=.8] (k) at (120:1cm) {$$};
\node[rectangle,draw=black,fill=white] (l) at (180:1cm) {$$};
\node[regular polygon,regular polygon sides=3,draw,scale=0.6] (m) at (240:1cm) {$$};
\node[rectangle,draw=black,fill=white] (n) at (300:1cm) {$$};
\draw[-]
(i) edge (j) (j) edge (k) (k) edge (l) (l) edge (m) (m) edge (n) (n) edge (i);
\begin{pgfonlayer}{background}
\draw[edge] (0:1cm) edge (120:1cm);
\draw[edge] (180:1cm) edge (300:1cm);
\end{pgfonlayer}
\end{tikzpicture}
\caption*{G}
\end{subfigure}
\Large$\myrightarrow{\hspace*{.25cm}\phi\hspace*{.25cm}}$
\begin{subfigure}[c]{.2\textwidth}\centering
\begin{tikzpicture}[scale=.6]
\node[circle,draw=black,fill=white,scale=.6] (A) at (0, 3) {$$};
\node[diamond,draw=black,fill=white,scale=.5] (B) at (0, 2) {$$};
\node[rectangle,draw=black,fill=white,scale=0.7] (C) at (0, 1) {$$};
\node[regular polygon,regular polygon sides=3,draw,scale=0.4] (E) at (0, 0) {$$};
\draw[-]  
(A) edge (B) (B) edge (C) (C) edge (E);
\end{tikzpicture}
\caption*{F}
\end{subfigure}
\caption{A forest approximation of $C_6$}
\label{evencycle}
\end{minipage}%
\begin{minipage}{.5\textwidth}
\centering
\begin{subfigure}[c]{.4\textwidth}\centering
\begin{tikzpicture}
\node[rectangle,draw=black,fill=white] (i) at (0:1cm) {$$};
\node[circle,draw=black,fill=white] (j) at (72:1cm) {$$};
\node[circle,draw=black,fill=white] (k) at (144:1cm) {$$};
\node[rectangle,draw=black,fill=white] (l) at (216:1cm) {$$};
\node[regular polygon,regular polygon sides=3,draw,scale=0.6] (m) at (288:1cm) {$$};
\draw
(i) edge (j) (j) edge (k) (k) edge (l) (l) edge (m) (m) edge (i);
\begin{pgfonlayer}{background}
\draw[edge] (0:1cm) edge (216:1cm);
\draw[edge] (72:1cm) edge (144:1cm);
\end{pgfonlayer}
\end{tikzpicture}
\caption*{G}
\end{subfigure}
\Large$\myrightarrow{\hspace*{.25cm}\phi\hspace*{.25cm}}$
\begin{subfigure}[c]{.2\textwidth}\centering
\begin{tikzpicture}[scale=.6]
\node[circle,draw=black,fill=white,scale=.6] (A) at (0, 3) {$$};
\node[rectangle,draw=black,fill=white,scale=0.7] (C) at (0, 1.5) {$$};
\node[regular polygon,regular polygon sides=3,draw,scale=0.4] (E) at (0, 0) {$$};
\draw[-]  
(A) edge (C) (C) edge (E);
\end{tikzpicture}
\caption*{F}
\end{subfigure}
\caption{A forest approximation of $C_5$}
\label{oddcycle}
\end{minipage}%
\end{figure}

\begin{ex}[$G$ is a grid] Suppose $G$ is a two-dimensional $(m\times m)$-grid. Then $n=m^2$. Considering the forest approximation illustrated in Figure~\ref{grid}, $\Lambda(G) \le 2 [ 3^2 +5^2 + \ldots + (2m - 1)^2 ] + 1
= \frac{2m(2m+1)(2m-1)-3}{3} 
= O(m^3) 
= O(n^\frac{3}{2}) 
$
\label{example grid}
\end{ex}

\begin{figure}[H]
% \vspace{-.8cm}
\centering
\begin{subfigure}{.45\textwidth}
\begin{tikzpicture}\centering
\draw (0, 0) grid [step=1] (3, 3);

\node [star,draw,scale=0.8,fill=white]  (00) at (0, 0) {};
\node [diamond,draw=black,fill=white,scale=.8]  (33) at (3, 3) {};

\node [regular polygon,regular polygon sides=5,draw,scale=0.8,fill=white]  (01) at (0, 1) {};
\node [regular polygon,regular polygon sides=5,draw,scale=0.8,fill=white]  (10) at (1, 0) {};

\node [rectangle,draw=black,fill=white]  (02) at (0, 2) {};
\node [rectangle,draw=black,fill=white]  (11) at (1, 1) {};
\node [rectangle,draw=black,fill=white]  (20) at (2, 0) {};

\node [regular polygon,regular polygon sides=3,draw,scale=0.6,fill=white]  (13) at (1, 3) {};
\node [regular polygon,regular polygon sides=3,draw,scale=0.6,fill=white]  (22) at (2, 2) {};
\node [regular polygon,regular polygon sides=3,draw,scale=0.6,fill=white]  (31) at (3, 1) {};

\node [regular polygon,regular polygon sides=6,draw,scale=0.8,fill=white]  (23) at (2, 3) {};
\node [regular polygon,regular polygon sides=6,draw,scale=0.8,fill=white]  (32) at (3, 2) {};

\node [circle,draw=black,fill=white]  (03) at (0, 3) {};
\node [circle,draw=black,fill=white]  (12) at (1, 2) {};
\node [circle,draw=black,fill=white]  (21) at (2, 1) {};
\node [circle,draw=black,fill=white]  (30) at (3, 0) {};

\node []  () at (-2.5, 1.5) {$G$};
\begin{pgfonlayer}{background}
\draw[edge] (2,3) edge (3,2);

\draw[edge] (1,3) edge (3,1);

\draw[edge] (0,3) edge (3,0);

\draw[edge] (0,2) edge (2,0);

\draw[edge] (0,1) edge (1,0);
\end{pgfonlayer}
\end{tikzpicture}
\end{subfigure}

\vspace{-.1cm}
\begin{subfigure}[c]{.45\textwidth}
\begin{tikzpicture}
\node[] (i) at (-4.2,0) {$$};
\node[] (i) at (0,0) {$$};
\node[] (k) at (.28,-.4) {$\phi$};
\node[] (j) at (0,-.8) {$$};
\draw[->, thick]
(i) -- (j);
\end{tikzpicture}
\end{subfigure}
\vspace{-.1cm}

\begin{subfigure}[c]{.45\textwidth}
\begin{tikzpicture}
\centering
\node[] (0) at (-1, 0) {$F$};
\node[star,draw,scale=0.8,fill=white] (A) at (0, 0) {$$};
\node[regular polygon,regular polygon sides=5,draw,scale=0.8,fill=white] (B) at (1, 0) {$$};
\node[rectangle,draw=black,fill=white] (C) at (2, 0) {$$};
\node[circle,draw=black,fill=white] (D) at (3, 0) {$$};
\node[regular polygon,regular polygon sides=3,draw,scale=0.6,fill=white] (E) at (4, 0) {$$};
\node[regular polygon,regular polygon sides=6,draw,scale=0.8,fill=white] (F) at (5, 0) {$$};
\node[diamond,draw=black,fill=white,scale=.8] (G) at (6, 0) {$$};
\draw[-]  
(A) edge (B) (B) edge (C) (C) edge (D) (D) edge (E) (E) edge (F) (F) edge (G);
\end{tikzpicture}
\end{subfigure}
\caption{A forest approximation of the $(4\times 4)$-gird}
\label{grid}
\end{figure}
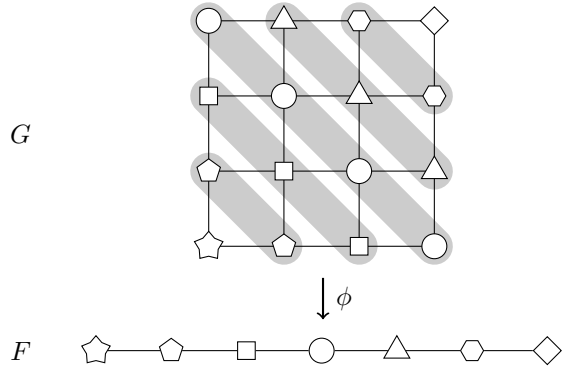

\section{Generalization Theory for Learning from Graph-Dependent Data}

This section establishes stability generalization error bounds for learning from graph-dependent data, using the concentration inequalities derived in the last section.

% Consider the supervised learning setting:
% The data universe is $\mathcal X \times \mathcal Y$ with distribution $D$, where $\mathcal X$ is the input space and $\mathcal Y$ is the output space. $\S = ((x_1, y_1), \ldots, (x_n, y_n)) \in (\mathcal X \times \mathcal Y)^n$ is a training sample of size $n$. 
% Assume that the data points $(x_i, y_i)$'s have the same marginal distribution $D$ 
% and the dependency graph $G$ of $\S$ is given.
% % With abuse of notation, we use $\S$ for both training sample and its realization for simplicity.
% % The advantage of using dependency graph is to capture the \textit{local dependence} among samples,
% % a single data can be dependent on at most $\Delta$ other samples.

Consider the supervised learning setting:
Let 
$\S = ((x_1, y_1), \ldots, (x_n, y_n)) \in (\mathcal X \times \mathcal Y)^n$ be a training sample of size $n$,
%  of size $n$,
where $\mathcal X$ is the input space and $\mathcal Y$ is the output space.
Let $D$ be the underlying distribution of data on $\mathcal X \times \mathcal Y$.
Assume that all the training data points $(x_i, y_i)$'s have the same marginal distribution $D$ 
and that $G$ is a dependency graph of $\S$.
%We call such samples \textit{graph-dependent} samples.

Throughout this section, fix a non-negative loss function
$\ell: \mathcal Y \times \mathcal Y \rightarrow \R$. 
% and suppose $(x, y)$ follows the distribution $D$.
For any hypothesis $f: \mathcal X \rightarrow \mathcal Y$, 
the empirical error on sample $\S$ is 
\[
\widehat R(f) = \frac{1}{n} \sum_{i = 1}^n \ell(y_i, f(x_i)). 
\]
For learning from dependent data, the generalization error can be defined in various ways. 
We adopt the following widely-used one ~\cite{meir2000nonparametric,lozano2006convergence,steinwart2009fast,hang2014fast}
\begin{equation}
R(f) = \E_{(x, y)\sim D} [ \ell(y, f(x)) ],
\label{gen error}
\end{equation}
% we give discussions in the last section regarding other definitions.
which assumes that the test set is independent of the training set.

\subsection{Bounding Generalization Error via Algorithmic Stability}
Algorithmic stability 
has been used in the study of classification and regression to derive generalization bounds ~\cite{rogers1978finite,devroye1979distribution,kearns1999algorithmic,kutin2002almost,mou2018dropout,mou2017generalization}.
A key advantage of stability bounds is that they are designed for specific learning algorithms, exploiting particular properties of the algorithms. Introduced 17 years ago, uniform stability \cite{bousquet2002stability} is now among the most widely used notions of algorithmic stability.

%The stability is usually studied in the context of the following perturbation of training data. 
Given a training sample $\S$ of size $n$ and $i\in [n]$, 
remove the $i$-th element from $\S$, resulting in a sample of size $n-1$, which is denoted by
$\S^{\setminus i} = ( (x_1, y_1), \ldots, (x_{i - 1}, y_{i - 1}), (x_{i + 1}, y_{i + 1}) \ldots, (x_n, y_n))$.
For a learning algorithm $\mathcal A$, 
define $f^{\mathcal A}_{\S} : \mathcal X \rightarrow \mathcal Y$ to be the
the hypothesis that $\mathcal A$ has learned from the sample $\S$.

\begin{df}[Uniform Stability~\cite{bousquet2002stability}]
Given integer $n>0$, the learning algorithm $\mathcal A$ is called $\beta_n$-uniformly stable with respect to the loss
function $\ell$, if for any $i\in [n]$, $\S \in (\mathcal X \times \mathcal Y)^n$, and $(x, y) \in \mathcal X \times \mathcal Y$,
it holds that
\[
| \ell(y, f^{\mathcal A}_{\S}(x)) - \ell(y, f^{\mathcal A}_{\S^{\setminus i}}(x))  | \le \beta_n.
\]
\end{df}
Intuitively, the stability of a leaning algorithm means that any small perturbation of training samples has little effect on the result of learning. 

Now, we begin our analysis with studying the distribution of $\Phi_{\mathcal A}(\S) = R(f^{\mathcal A}_{\S}) - \widehat R (f^{\mathcal A}_{\S})$, namely, the difference between the empirical and the generalization errors. The mapping $\Phi_{\mathcal A}: (\mathcal X \times \mathcal Y)^n \rightarrow \R$  will play a critical role in estimating $R(f^{\mathcal A}_{\S})$ via stability. 
We first show that the deviation of $\Phi_{\mathcal A}(\S)$ from its expectation can be bounded with high probability (Lemma \ref{stability concentration}), and then upper bound the expected value of $\Phi_{\mathcal A}(\S)$ in Lemma \ref{stability expectation lemma via degree}. 
% Stability plays an important role in both steps.

% The following property of $\Phi(\cdot)$ was proved in~\cite{bousquet2002stability} for independent data, but that proof remains valid in our setting. Hence we restate it without proof. 

% \begin{lm}
% Given a graph-dependent sample $\S$ of size $n$ with its dependency graph $G$,
% let $\A$ be a $\beta_n$-uniformly stable learning algorithm that outputs a function $f_\S$.
% For any $\S, \S' \in (\mathcal X \times \mathcal Y)^n$ that differ only in one entry, 
% the following holds:
% \[
% | \Phi(\S) - \Phi(\S')| \le 4\beta_n + \frac{M}{n}
% \]
% \label{generalization error bound bounded difference}
% \end{lm}
% \begin{proof}
% In~\cite{bousquet2002stability}, Lemma~\ref{generalization error bound bounded difference}
% was proved for i.i.d. data, actually,
% the proof remains valid in our setting.
% We provide the proof in the supplementary materials.
% \end{proof}

\begin{lm}
Given a sample $\S$ of size $n$ with dependency graph $G$, 
assume that the learning algorithm $\A$ is $\beta_n$-uniformly stable. Suppose the
loss function $\ell$ is bounded by $M$.
Then for any $t>0$, it holds that
\[
\P( \Phi_{\mathcal A}(\S) - \E [\Phi_{\mathcal A}(\S)] \ge t )
\le \exp \left( - \dfrac{2n^2t^2}{\Lambda(G) (4n\beta_n + M)^2} \right).
\]
\label{stability concentration}
\end{lm}
Lemma \ref{stability concentration} is proved in two steps. First, we treat $\Phi_{\mathcal A}(\cdot)$ as an $n$-ary function and show that its Lipschitz coefficients are all bounded by $4\beta_n + M/n$. Second, regarding $\S$ as a random vector, we apply Theorem \ref{Concentration Inequality for Dependency Forest 1} to $\Phi_{\mathcal A}(\S)$.
For detail, see Subsection \ref{Phiisconcentrated} of the supplementary materials.

% \begin{proof}
% We give the proof of the following Lipschitz property in the supplementary materials,
% which was proved under the i.i.d. assumption in~\cite{bousquet2002stability}
% \[
% | \Phi(\S) - \Phi(\S')| \le 4\beta_n + \frac{M}{n}
% \]
% Then we apply
% Theorem~\ref{Concentration Inequality for Dependency Forest 1}.
% \end{proof}

\begin{lm}
Given a sample $\S$ of size $n$ with dependency graph $G$,
assume that the learning algorithm $\A$ is $\beta_i$-uniformly stable for any $i\le n$.
% , and $f_\S$ be the
% the hypothesis that $\mathcal A$ has learned from the sample $\S$.
Suppose the maximum degree of $G$ is $\Delta$. Let $\beta_{n, \Delta} = \max_{i \in [0,\Delta]} \beta_{n-i}$. It holds that
\[
\E [ \Phi_{\mathcal A}(\S) ] \le 2 \beta_{n, \Delta} (\Delta + 1).
\]
\label{stability expectation lemma via degree}
\end{lm}
The proof of the lemma is based on iterative perturbations on the training sample $\S$. A perturbation is essentially removing a data point from or adding a data point to $\S$. The property of uniform stability of the algorithm guarantees that each perturbation causes a discrepancy up to $\beta_{n,\Delta}$, and in total $2(\Delta +1)$ perturbations have to be made in order to \textit{eliminate} the dependency between a data point and the others. 
% quite involving, but one can easily see that the result is reasonable. In the special case that $\S$ is i.i.d., $\Delta=0$, and the bound exactly degenerates to the classic i.i.d. result.  is relies 
% on the properties of uniform stability and the dependency graph. 
% Note that one individual random variable depends on at most $\Delta$ others. 
% Suppose we iteratively remove the a single data with its neighbors from the graph-dependent sample,
% it will result in accumulation of $\beta$ for at most $(\Delta + 1)$ times.
% Since uniform stability is defined for a sample of fixed size, 
% we use the maximum stability coefficient during the removing process.
For detail, please refer to Subsection \ref{Phihassmallexpectation} of the supplementary materials.

Combining Lemma~\ref{stability concentration} and Lemma~\ref{stability expectation lemma via degree}, we immediately have
\begin{tm}
Given a sample $\S$ of size $n$ with dependency graph $G$, assume that the learning algorithm $\A$ is $\beta_i$-uniformly stable  for any $i\le n$.
Suppose the maximum degree $G$ is $\Delta$, and the 
loss function $\ell$ is bounded by $M$. 
Let $\beta_{n, \Delta} = \max_{i \in [0,\Delta]} \beta_{n-i}$.
For any $\delta \in (0, 1)$, with probability at least $1 - \delta$,
it holds that
\[
R(f^{\mathcal A}_{\S}) 
\le \widehat R (f^{\mathcal A}_{\S}) + 2 \beta_{n, \Delta} (\Delta + 1)
+ \frac{4n\beta_n + M}{n} \sqrt{\frac{\Lambda(G)\ln(1/\delta)}{2} }.
\]
\label{gen bound via sta}
\end{tm}

\begin{re}
%For most of the existing learning algorithms, 
It is well known that for many learning algorithms $\beta_n=O(1/n)$ \cite{bousquet2002stability}.
Thus, we often have 
$\beta_{n, \Delta} (\Delta + 1) 
\le \beta_{n-\Delta} (\Delta + 1) = O(\frac{\Delta}{n - \Delta})$,
which vanishes asymptotically if $\Delta = o(n)$.
The term $O\left(\sqrt{\Lambda(G)}/n\right)$ also vanishes asymptotically if 
$\Lambda(G)=o(n^2)$. As a result, in case of \textit{weak} dependence such as the examples in Subsection~\ref{Examples and Applications}, the generalization error is almost upper-bounded by the empirical error. We also observe that if the training data are i.i.d., Theorem~\ref{gen bound via sta} degenerates to the standard stability bound 
in~\cite{bousquet2002stability}, by applying $\Delta = 0$, $\beta_{n, \Delta} = \beta_n$, $\Lambda(G) = n$. 

%The estimation and existing results of $\beta_n$ for specific learning algorithms 
%lead to generalization bounds in the particular cases.
\label{gen bound analysis}
\end{re}

\subsection{Application: Learning from $m$-dependent data}

We present a practical application in machine learning. 
Suppose there are linearly aligned locations, for example, real estates along a street. Let $y_i$ be the observation at location $i$, e.g., the house price, and $x_i$ stand for the random variable modeling geographical effect
at location $i$. Suppose that $x$'s are mutually independent and each $y_i$ is geographically influenced by a neighborhood of size at most $2q+1$. One hope to learn the model of $y$ from a sample 
$ \{( ( x_{i-q},\ldots,x_i,\ldots,x_{i+q} ), y_i )\}_{i\in [n]}$, where $n$ is the size of the sample.
This model accounts for the impact of local locations on house prices. 
Similar scenarios are frequently considered in spatial econometrics,
see~\cite{anselin2013spatial} for more examples.

This application is a special case of $m$-dependence, which is an important statistical model introduced by Hoeffding in~\cite{hoeffding1948central}. $m$-dependence has been studied extensively in probability, statistics, and combinatorics
\cite{diananda1953some,sen1968asymptotic,chen2005stein}.

\begin{df}[$m$-dependence~\cite{hoeffding1948central}]
For some $m, n \in \mathbb N_+$, a sequence of random variables $\{X_i\}_{i = 1}^n$ is called $m$-dependent if for any $i \in [n-m - 1]$, 
$\{X_j\}_{j = 1}^i$ is independent of $\{X_j\}_{j = i+m+1}^n$. 
\label{m-dependence result}
\end{df}

The upper part of Figure~\ref{chain} illustrates a dependency graph of $2$-dependent sequence $\{X_i\}_{i = 1}^n$. 

% Consider the forest approximation in Figure~\ref{chain}. 
As illustrated in Figure~\ref{chain}, we divide an $m$-dependent sequence into blocks of size $m$, and sequentially map the blocks to vertices of a path of length $\left\lceil\frac{n}{m}\right\rceil$. This forest approximation leads to
\[
\Lambda(G)
\le \left(\left\lceil\frac{n}{m}\right\rceil- 1\right)( m+m )^2 + m^2 
\le 4mn = O(mn)
\]

\begin{figure}[H] 
\centering
\begin{subfigure}[c]{.45\textwidth}
\begin{tikzpicture}[scale=.8]
\node[] (0) at (-1, 0) {$G$};
\node[circle,draw=black,fill=white] (A) at (0, 0) {$$};
\node[circle,draw=black,fill=white] (B) at (1.5, 0) {$$};
\node[rectangle,draw=black,fill=white] (C) at (3, 0) {$$};
\node[rectangle,draw=black,fill=white] (D) at (4.5, 0) {$$};
\node[regular polygon,regular polygon sides=3,draw,scale=0.6,fill=white] (E) at (6, 0) {$$};
\node[regular polygon,regular polygon sides=3,draw,scale=0.6,fill=white] (F) at (7.5, 0) {$$};
\draw[out=-60, in=-120]  
(A) edge (B) (B) edge (C) (C) edge (D) (D) edge (E) (E) edge (F);
\draw[out=-90, in=-90]
(A) edge (C) (B) edge (D) (C) edge (E) (D) edge (F);
\begin{pgfonlayer}{background}
\draw[edge] (0, 0) edge (1.5, 0);
\draw[edge] (3, 0) edge (4.5, 0);
\draw[edge] (6, 0) edge (7.5, 0);
\end{pgfonlayer}
\end{tikzpicture}
\end{subfigure}

\vspace{-.2cm}
\begin{subfigure}[c]{.4\textwidth}
\begin{tikzpicture}
\node[] (i) at (-3.55,0) {$$};
\node[] (i) at (0,0) {$$};
\node[] (k) at (.28,-.4) {$\phi$};
\node[] (j) at (0,-.8) {$$};
\draw[->, thick]
(i) -- (j);
\end{tikzpicture}
\end{subfigure}
% \vspace{-.cm}

\begin{subfigure}[c]{.45\textwidth}
\begin{tikzpicture}[scale=.6]
\node[] (0) at (-4.35, 0) {$F$};
\node[circle,draw=black,fill=white] (A) at (0, 0) {$$};
\node[rectangle,draw=black,fill=white] (C) at (2, 0) {$$};
\node[regular polygon,regular polygon sides=3,draw,scale=0.6] (E) at (4, 0) {$$};
\draw[-]  
(A) edge (C) (C) edge (E);
\end{tikzpicture}
\end{subfigure}
\caption{A forest approximation of a $2$-dependent sequence. The approximation $\phi$ maps any vertex of $G$ to the vertex of $F$ having the same shape, so each gray belt stands for a pre-image set of $\phi$.}
\label{chain}
\end{figure}
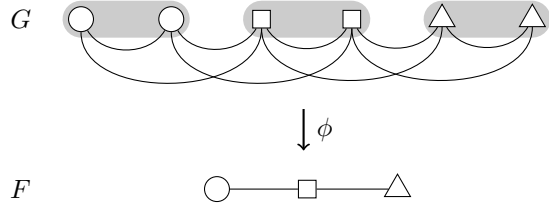

Combining Theorem~\ref{gen bound via sta} and the estimated forest complexity, we have

\begin{co}
Given an $m$-dependent sequence $\S$ of length $n$ as training sample,
assume that the learning algorithm $\A$ is $\beta_i$-uniformly stable  for any $i\le n$.
Suppose the 
loss function $\ell$ is bounded by $M$. 
For any $\delta \in (0, 1)$,
with probability at least $1 - \delta$,
it holds that
\[
R(f^{\A}_{\S}) 
\le \widehat R (f^{\A}_{\S}) + 2 \beta_{n, 2m} (2m + 1)
+ ( 4n\beta_n + M ) \sqrt{\frac{2 m\ln(1/\delta)}{ n } }.
\]
\label{learning m dependent}
\end{co}

% \begin{ex}[Moving-average model]
% $\{X_i\}$ is a moving-average model MA($q$) of order $q$ if
% \[
% X_t = \mu + \varepsilon_t + \sum_{i=1}^q \theta_i \varepsilon_{t-i}
% \]
% where $\mu$ is the mean of the series, 
% $\{\theta_i\}$ are the parameters of the model
% and $\{\varepsilon_i\}$ are i.i.d. random variables.
% MA($q$) model is a sequence of $q$-dependent random variables.
% Thus, Theorem~\ref{learning m dependent} is valid for learning MA($q$) model
% in a supervised learning setting, 
% although the traditional time series analysis focuses on estimation of parameters.
% \label{ma}
% \end{ex}

Choose any uniformly stable learning algorithm $\A$ in~\cite{bousquet2002stability} with $\beta_n = O(1/n)$,
such as regularization algorithms in  RKHS. Apply it to the above mentioned house price prediction problem. Then for any fixed $q$, with high probability, Corollary~\ref{learning m dependent} leads to
$R(f^{\A}_{\S}) 
\le \widehat R (f^{\A}_{\S}) + O\left(\sqrt{\frac{\ln(1/\delta)}{ n } }\right)$ 
for sufficiently large $n$, matching the stability bound of the i.i.d.\ case in \cite{bousquet2002stability}.

% We give a practical application in machine learning which follows the model of $m$-dependence. 
% %by spatial moving-average model~\cite{anselin2013spatial}.
% Suppose there are linearly aligned locations, for example, real estates along a street. Let $y_i$ is the observation at location $i$, such as housing price. Let 
% $\{x_i\}$ stand for the random variables modeling geographical effect
% of locations. Given $q$ and $n$, 
% $ \{ ( x_{i-q},\ldots,x_i,\ldots,x_{i+q} ), y_i \}_{i\in [n]}$ 
% can be used as a $2q$-dependent sample for learning.
% This model accounts for the impact of local locations on housing prices.
% Similar models are commonly used in spatial econometrics,
% see~\cite{anselin2013spatial} for details.
% One can apply uniformly stable learning algorithms in~\cite{bousquet2002stability}
% such as RKHS learning with regularizer, whose $\beta_n = O(1/n)$. For any fixed $q$, Theorem~\ref{learning m dependent} leads to
% $R \le \widehat R + O\left(\sqrt{\frac{\ln(1/\delta)}{ n } }\right)$ 
% for sufficiently large $n$, matching the bound in the i.i.d. case.

\section{Conclusion and Future Work}
In this paper, we establish McDiarmid-type concentration inequalities for general functions of graph-dependent random variables.
We apply our concentration results to obtain a stability-based generalization error bound for learning from graph-dependent samples.
There are several possible extensions of this work. 
\begin{itemize}
\item \label{gamma} We provide upper bounds of the forest complexity 
for several classes of graphs.
It is an interesting algorithmic problem to efficiently estimate the forest complexity. One heuristic method to do this on a connected graph is via graph diameter, by merging vertices of the same distances to a peripheral vertex, resulting in a path as long as the diameter. Can the problem be solved approximately? 

\item If more information of the dependency structure is known, 
e.g., a dependency hypergraph~\cite{DBLP:conf/alt/WangGR17}, 
can we obtain better McDiarmid-type inequalities and tighter generalization bounds?

\item In \cite{mohri2008stability,mohri2010stability,kuznetsov2017generalization}, generalization error is defined different than that in this paper.
% We consider the generalization error defined in \eqref{gen error}, while
% another definition is presented in
%\cite{mohri2008stability,mohri2010stability,kuznetsov2017generalization}. %is 
%$ \E [ \ell(y, f(x)) | \S ] $,
%which assumes the test point $(x, y)$ is evaluated given training data $\S$. 
The relationship between these two definitions has been discussed in \cite{mohri2008stability,mohri2010stability}. 
It is a natural question whether our results can be adapted to that definition. 

\item There are some newly introduced dependency graph models such as
thresholded dependency graphs~\cite{lampert2018dependency} 
and weighted dependency graphs~\cite{dousse2016weighted,feray2018weighted}.
Can the problem in this paper be solved under these new models?

% \item The learnability of non-i.i.d. has been studied under various mixing conditions~
% \cite{vidyasagar2013learning,gao2016learnability},
% We have showed that $m$-dependent data (in Example~\ref{m-dependence result}), 
% which exhibits weak dependence is learnable in Theorem~\ref{learning m dependent}.
% However, if the dependenc graph is a completer graph, $\Lambda(G)=n^2$, 
% the worst case is that a single sample is repeated, in which generalization is impossible.
% It is also a direction to establish results for learnability of graph-dependent data,
% this will be close connected to the first point in this section.

\end{itemize}

\subsubsection*{Acknowledgments}

Rui (Ray) Zhang would like to thank Nick Wormald for valuable comments on an early version of this paper. Yuyi Wang would like to thank Dr.\ Ond\v{r}ej Ku\v{z}elka for very helpful discussions. Liwei Wang would like to thank Yunchang Yang for very helpful discussions. 
Xingwu Liu's work is partially supported by the National Key Research and Development Program of China (Grant No. 2016YFB1000201), the National Natural Science Foundation of China (61420106013), State Key Laboratory of Computer Architecture Open Fund (CARCH3410), and Youth Innovation Promotion Association of Chinese Academy of Sciences.

\bibliographystyle{unsrt}
\bibliography{nips}

\begin{thebibliography}{10}

\bibitem{dehling2002empirical}
Herold Dehling and Walter Philipp.
\newblock Empirical process techniques for dependent data.
\newblock In {\em Empirical process techniques for dependent data}, pages
  3--113. Springer, 2002.

\bibitem{amini2015learning}
Massih-Reza Amini and Nicolas Usunier.
\newblock {\em Learning with Partially Labeled and Interdependent Data}.
\newblock Springer, 2015.

\bibitem{mohri2008stability}
Mehryar Mohri and Afshin Rostamizadeh.
\newblock Stability bounds for non-iid processes.
\newblock In {\em Advances in Neural Information Processing Systems}, pages
  1025--1032, 2008.

\bibitem{mohri2009rademacher}
Mehryar Mohri and Afshin Rostamizadeh.
\newblock Rademacher complexity bounds for non-iid processes.
\newblock In {\em Advances in Neural Information Processing Systems}, pages
  1097--1104, 2009.

\bibitem{ralaivola2010chromatic}
Liva Ralaivola, Marie Szafranski, and Guillaume Stempfel.
\newblock Chromatic pac-bayes bounds for non-iid data: Applications to ranking
  and stationary $\beta$-mixing processes.
\newblock {\em Journal of Machine Learning Research}, 11(Jul):1927--1956, 2010.

\bibitem{kuznetsov2017generalization}
Vitaly Kuznetsov and Mehryar Mohri.
\newblock Generalization bounds for non-stationary mixing processes.
\newblock {\em Machine Learning}, 106(1):93--117, 2017.

\bibitem{yi2018learning}
Hao Yi, Alon Orlitsky, and Venkatadheeraj Pichapati.
\newblock On learning markov chains.
\newblock In {\em Advances in Neural Information Processing Systems}, pages
  646--655, 2018.

\bibitem{rosenblatt1956central}
Murray Rosenblatt.
\newblock A central limit theorem and a strong mixing condition.
\newblock {\em Proceedings of the National Academy of Sciences of the United
  States of America}, 42(1):43, 1956.

\bibitem{volkonskii1959some}
VA~Volkonskii and Yu~A Rozanov.
\newblock Some limit theorems for random functions. i.
\newblock {\em Theory of Probability \& Its Applications}, 4(2):178--197, 1959.

\bibitem{ibragimov1962some}
Ildar~A Ibragimov.
\newblock Some limit theorems for stationary processes.
\newblock {\em Theory of Probability \& Its Applications}, 7(4):349--382, 1962.

\bibitem{kontorovich2007measure}
Leonid Kontorovich.
\newblock {\em Measure concentration of strongly mixing processes with
  applications}.
\newblock Carnegie Mellon University, 2007.

\bibitem{mohri2010stability}
Mehryar Mohri and Afshin Rostamizadeh.
\newblock Stability bounds for stationary $\varphi$-mixing and $\beta$-mixing
  processes.
\newblock {\em Journal of Machine Learning Research}, 11(Feb):789--814, 2010.

\bibitem{he2016stability}
Fangchao He, Ling Zuo, and Hong Chen.
\newblock Stability analysis for ranking with stationary $\varphi$-mixing
  samples.
\newblock {\em Neurocomputing}, 171:1556--1562, 2016.

\bibitem{kontorovich2008concentration}
Leonid~Aryeh Kontorovich, Kavita Ramanan, et~al.
\newblock Concentration inequalities for dependent random variables via the
  martingale method.
\newblock {\em The Annals of Probability}, 36(6):2126--2158, 2008.

\bibitem{yu1994rates}
Bin Yu.
\newblock Rates of convergence for empirical processes of stationary mixing
  sequences.
\newblock {\em The Annals of Probability}, pages 94--116, 1994.

\bibitem{pmlr-v99-dagan19a}
Yuval Dagan, Constantinos Daskalakis, Nishanth Dikkala, and Siddhartha Jayanti.
\newblock Learning from weakly dependent data under dobrushin’s condition.
\newblock In {\em Proceedings of the Thirty-Second Conference on Learning
  Theory}, volume~99 of {\em Proceedings of Machine Learning Research}, pages
  914--928, Phoenix, USA, 25--28 Jun 2019. PMLR.

\bibitem{kulske2003concentration}
Christof K{\"u}lske.
\newblock Concentration inequalities for functions of gibbs fields with
  application to diffraction and random gibbs measures.
\newblock {\em Communications in mathematical physics}, 239(1-2):29--51, 2003.

\bibitem{chatterjee2005concentration}
Sourav Chatterjee.
\newblock Concentration inequalities with exchangeable pairs (Ph. D. thesis).
\newblock {\em arXiv preprint math/0507526}, 2005.

\bibitem{kontorovich2017concentration}
Aryeh Kontorovich and Maxim Raginsky.
\newblock Concentration of measure without independence: a unified approach via
  the martingale method.
\newblock In {\em Convexity and Concentration}, pages 183--210. Springer, 2017.

\bibitem{dobruschin1968description}
PL~Dobruschin.
\newblock The description of a random field by means of conditional
  probabilities and conditions of its regularity.
\newblock {\em Theory of Probability \& Its Applications}, 13(2):197--224,
  1968.

\bibitem{janson2004large}
Svante Janson.
\newblock Large deviations for sums of partly dependent random variables.
\newblock {\em Random Structures \& Algorithms}, 24(3):234--248, 2004.

\bibitem{usunier2006generalization}
Nicolas Usunier, Massih-Reza Amini, and Patrick Gallinari.
\newblock Generalization error bounds for classifiers trained with
  interdependent data.
\newblock In {\em Advances in neural information processing systems}, pages
  1369--1376, 2006.

\bibitem{ralaivola2015entropy}
Liva Ralaivola and Massih-Reza Amini.
\newblock Entropy-based concentration inequalities for dependent variables.
\newblock In {\em International Conference on Machine Learning}, pages
  2436--2444, 2015.

\bibitem{DBLP:conf/alt/WangGR17}
Yuyi Wang, Zheng{-}Chu Guo, and Jan Ramon.
\newblock Learning from networked examples.
\newblock In {\em International Conference on Algorithmic Learning Theory,
  {ALT} 2017, 15-17 October 2017, Kyoto University, Kyoto, Japan}, pages
  641--666, 2017.

\bibitem{mcdiarmid1989method}
Colin McDiarmid.
\newblock On the method of bounded differences.
\newblock {\em Surveys in combinatorics}, 141(1):148--188, 1989.

\bibitem{erdos1975problems}
Paul Erdos and L{\'a}szl{\'o} Lov{\'a}sz.
\newblock Problems and results on 3-chromatic hypergraphs and some related
  questions.
\newblock {\em Infinite and finite sets}, 10(2):609--627, 1975.

\bibitem{janson1988exponential}
Svante Janson, Tomasz Luczak, and Andrzej Rucinski.
\newblock {\em An exponential bound for the probability of nonexistence of a
  specified subgraph in a random graph}.
\newblock Institute for Mathematics and its Applications (USA), 1988.

\bibitem{chen1978two}
Louis~HY Chen.
\newblock Two central limit problems for dependent random variables.
\newblock {\em Probability Theory and Related Fields}, 43(3):223--243, 1978.

\bibitem{baldi1989normal}
Pierre Baldi, Yosef Rinott, et~al.
\newblock On normal approximations of distributions in terms of dependency
  graphs.
\newblock {\em The Annals of Probability}, 17(4):1646--1650, 1989.

\bibitem{janson2011random}
Svante Janson, Tomasz Luczak, and Andrzej Rucinski.
\newblock {\em Random graphs}, volume~45.
\newblock John Wiley \& Sons, 2011.

\bibitem{linderman2014discovering}
Scott Linderman and Ryan Adams.
\newblock Discovering latent network structure in point process data.
\newblock In {\em International Conference on Machine Learning}, pages
  1413--1421, 2014.

\bibitem{kirichenko2015optimality}
Alisa Kirichenko and Harry Van~Zanten.
\newblock Optimality of poisson processes intensity learning with gaussian
  processes.
\newblock {\em The Journal of Machine Learning Research}, 16(1):2909--2919,
  2015.

\bibitem{meir2000nonparametric}
Ron Meir.
\newblock Nonparametric time series prediction through adaptive model
  selection.
\newblock {\em Machine learning}, 39(1):5--34, 2000.

\bibitem{lozano2006convergence}
Aur{\'e}lie~C Lozano, Sanjeev~R Kulkarni, and Robert~E Schapire.
\newblock Convergence and consistency of regularized boosting algorithms with
  stationary b-mixing observations.
\newblock In {\em Advances in neural information processing systems}, pages
  819--826, 2006.

\bibitem{steinwart2009fast}
Ingo Steinwart and Andreas Christmann.
\newblock Fast learning from non-iid observations.
\newblock In {\em Advances in neural information processing systems}, pages
  1768--1776, 2009.

\bibitem{hang2014fast}
Hanyuan Hang and Ingo Steinwart.
\newblock Fast learning from $\alpha$-mixing observations.
\newblock {\em Journal of Multivariate Analysis}, 127:184--199, 2014.

\bibitem{rogers1978finite}
William~H Rogers and Terry~J Wagner.
\newblock A finite sample distribution-free performance bound for local
  discrimination rules.
\newblock {\em The Annals of Statistics}, pages 506--514, 1978.

\bibitem{devroye1979distribution}
Luc Devroye and Terry Wagner.
\newblock Distribution-free performance bounds for potential function rules.
\newblock {\em IEEE Transactions on Information Theory}, 25(5):601--604, 1979.

\bibitem{kearns1999algorithmic}
Michael Kearns and Dana Ron.
\newblock Algorithmic stability and sanity-check bounds for leave-one-out
  cross-validation.
\newblock {\em Neural computation}, 11(6):1427--1453, 1999.

\bibitem{kutin2002almost}
Samuel Kutin and Partha Niyogi.
\newblock Almost-everywhere algorithmic stability and generalization error.
\newblock In {\em Proceedings of the Eighteenth conference on Uncertainty in
  artificial intelligence}, pages 275--282. Morgan Kaufmann Publishers Inc.,
  2002.

\bibitem{mou2018dropout}
Wenlong Mou, Yuchen Zhou, Jun Gao, and Liwei Wang.
\newblock Dropout training, data-dependent regularization, and generalization
  bounds.
\newblock In {\em International Conference on Machine Learning}, pages
  3642--3650, 2018.

\bibitem{mou2017generalization}
Wenlong Mou, Liwei Wang, Xiyu Zhai, and Kai Zheng.
\newblock Generalization bounds of sgld for non-convex learning: Two
  theoretical viewpoints.
\newblock {\em arXiv preprint arXiv:1707.05947}, 2017.

\bibitem{bousquet2002stability}
Olivier Bousquet and Andr{\'e} Elisseeff.
\newblock Stability and generalization.
\newblock {\em Journal of machine learning research}, 2(Mar):499--526, 2002.

\bibitem{anselin2013spatial}
Luc Anselin.
\newblock {\em Spatial econometrics: methods and models}, volume~4.
\newblock Springer Science \& Business Media, 2013.

\bibitem{hoeffding1948central}
Wassily Hoeffding, Herbert Robbins, et~al.
\newblock The central limit theorem for dependent random variables.
\newblock {\em Duke Mathematical Journal}, 15(3):773--780, 1948.

\bibitem{diananda1953some}
PH~Diananda and MS~Bartlett.
\newblock Some probability limit theorems with statistical applications.
\newblock In {\em Mathematical Proceedings of the Cambridge Philosophical
  Society}, volume~49, pages 239--246. Cambridge University Press, 1953.

\bibitem{sen1968asymptotic}
Pranab~Kumar Sen.
\newblock Asymptotic normality of sample quantiles for m-dependent processes.
\newblock {\em The annals of mathematical statistics}, pages 1724--1730, 1968.

\bibitem{chen2005stein}
Louis~HY Chen and Qi-Man Shao.
\newblock Stein’s method for normal approximation.
\newblock {\em An introduction to Stein’s method}, 4:1--59, 2005.

\bibitem{lampert2018dependency}
Christoph~H Lampert, Liva Ralaivola, and Alexander Zimin.
\newblock Dependency-dependent bounds for sums of dependent random variables.
\newblock {\em arXiv preprint arXiv:1811.01404}, 2018.

\bibitem{dousse2016weighted}
Jehanne Dousse and Valentin F{\'e}ray.
\newblock Weighted dependency graphs and the ising model.
\newblock {\em arXiv preprint arXiv:1610.05082}, 2016.

\bibitem{feray2018weighted}
Valentin F{\'e}ray et~al.
\newblock Weighted dependency graphs.
\newblock {\em Electronic Journal of Probability}, 23, 2018.

\end{thebibliography}

\newpage
\appendix
\section{Omitted Proofs in Section 3}
\label{Omitted Proofs in Section 3}

\subsection{Proof of Theorem~\ref{Concentration Inequality for Dependency Tree}}\label{tree}

Given a random vector $\X=(X_1,\ldots,X_n)$
taking values in a product space $\mathbf{\Omega}=\prod_{i\in [n]}\Omega_i$.
For any set $S\subseteq [n]$, 
we denote $\X_S=\{X_i\}_{i\in S}$, and $\mathbf{\Omega}_S=\prod_{i\in S}\Omega_i$ for convenience.
The proof of Theorem \ref{Concentration Inequality for Dependency Tree} will rest on Lemma \ref{McDiarmid's Lemma}, which intuitively means that the small deviation of 
\[
\E [ f(\X) | X_1=x_1,\ldots,X_{i-1}=x_{i-1},X_i=x_i]
\]
with respect to $x_i$ for all $i$ leads to a high concentration of $f(\X)$ around its expectation.
Our task is thus reduced to show that when $x_1,\ldots,x_{i-1}$ is fixed, 
\[
\E [ f(\X) | X_1=x_1,\ldots,X_{i-1}=x_{i-1},X_i=x_i]-\E [ f(\X) | X_1=x_1,\ldots,X_{i-1}=x_{i-1},X_i=x'_i]
\]
is small for any $x_i$ and $x_i'$. This will be true due to Lemma \ref{martingale differences lemma}, if there is a good coupling, namely, jointly distributed variables $(\mathbf{Y},\mathbf{Z})$ whose Hamming distance is small and whose marginals are the distributions of $\X$ conditional on 
$\{ X_1=x_1,\ldots,X_{i-1}=x_{i-1},X_i=x_i \}$ and on 
$\{ X_1=x_1,\ldots,X_{i-1}=x_{i-1},X_i=x'_i \}$, respectively. 
Hence, the main part of the proof is to construct such a coupling (Lemma \ref{coupling}) whose feasibility relies on the strong independence among $\X$ (Lemma \ref{lm:indepofxi}).
First of all, recall a lemma in literature. 

\begin{lm}[\cite{mcdiarmid1989method}]
If for any $j \in [n]$ and $\mathbf{y}\in \mathbf{\Omega}_{[j-1]}$,
there is $b_j\ge 0$ such that 
\begin{equation}
\sup_{\xi\in \Omega_j} \E [ f(\X) | \X_{[j-1]}= \y ,X_j=\xi ]
- \inf_{\xi\in \Omega_j} \E [ f(\X) | \X_{[j-1]}= \y ,X_j=\xi ] \le b_j 
\end{equation}
then for any $t > 0$,
\[
\P( f(\X) - \E  [f(\X)] \ge t )
\le \exp \left(  - \dfrac{2t^2}{ \sum_{j=1}^n b_j^2 }  \right). 
\]
\label{McDiarmid's Lemma}
\end{lm}

By this lemma, it suffice to show that the small deviation of 
$\E [ f(\X) | X_1=x_1,\ldots,X_{i-1}=x_{i-1},X_i=x_i]$ with respect to $x_i$ for all $i$ is small
to prove Theorem~\ref{Concentration Inequality for Dependency Tree}. 
Before continuing the proof, we assume that
\begin{description}
\item[Well-rooted:] $G$ is rooted at the vertex $n$ and $c_n=c_{\min}$. \label{Well-rooted}
\item[Well-sorted:] For any $i,j\in V(G)$, $j$ is a descendent of $i$ only if $j<i$. \label{ordered}
\end{description}
These assumptions will not lose generality, since we can relabel the variables $X_1,...,X_n$ to meet the requirements.

For any non-root vertex $i\in V(G)$, let $p(i)$ be the parent vertex of $i$.
For the rest of the section, arbitrarily fix $i\in [n]$ and define $S=[i+1,n]\setminus \{ p(i) \}$, where $[j,k]$ stands for the set $\{j,\ldots,k\}$ of integers. Arbitrarily choose a vector $\mathbf{x}=(x_1,\ldots,x_n)\in \mathbf{\Omega}$ and an element $x'_i\in \Omega_i$. Let 
$\mathbf{x'}=(x_1,\ldots,x_{i-1},x'_i,x_{i+1},\ldots,x_n)$. 
We have the following technical lemma, indicating that $\X_S$ is independent of $X_i$ if $\X_{[i-1]}$ is given.
\begin{lm}
For any vector $\mathbf{y}\in \mathbf{\Omega}_S$, 
\[
\P( \X_S=\mathbf{y} | \X_{[i]}=\mathbf{x}_{[i]})
= \P( \X_S=\mathbf{y} | \X_{[i]}=\mathbf{x}'_{[i]}).
\]
\label{lm:indepofxi}
\end{lm}
\begin{proof}
Let $T_i$ be the subtree of $G$ that is rooted at $i$. Our basic idea is to prove the stronger property that $\X_S$ is independent of the other parts of $\X_{[i]}$ if $\X_{[i-1]\setminus V(T_i)}$ is given. Since $[i]=V(T_i)\bigcup ([i-1]\setminus V(T_i))$, it suffices to show that $\X_{V(T_i)}$ is independent of $\{\X_{S},\X_{[i-1]\setminus V(T_i)}\}$, which in turn is reduced to prove the following two claims due to the definition of the dependency graphs.
\begin{description}
\item[Claim 1]: $N^+_G(T_i)\bigcap ([i-1]\setminus V(T_i))=\emptyset$, where $N^+_G(T_i)=\bigcup_{k\in V(T_i)}N^+_G(k)$. \\
Proof of Claim 1: Arbitrarily choose $j\in [i-1]\setminus V(T_i)$. Suppose for contradiction that $j\in N_G(k)$ for some $k\in V(T_i)$, namely, $j$ is either a child or the parent of $k$. Since $j\notin V(T_i)$, we must have $j=p(k)$ and $k=i$, which implies $j>i$ due to the Assumption Well-rooted. A contradiction is reached, so $N_G(T_i)\bigcap ([i-1]\setminus V(T_i)) =\emptyset$. Because $N^+_G(T_i)\bigcap ([i-1]\setminus V(T_i)) =N_G(T_i)\bigcap ([i-1]\setminus V(T_i))$, Claim 1 holds.
\item[Claim 2]: $N^+_G(T_i) \bigcap S=\emptyset$. 
\\
Proof of Claim 2: Arbitrarily choose $j\in S=\{i+1,\ldots,n\}\setminus N_G(i)$. One immediately has $j\notin N^+_G(i)$. Suppose for contradiction that $j\in N^+_G(T_i)$. Then $j\in N^+_G(k)$ for some descendent $k$ of $i$, which means that either $j=i$ or $j$ is a descendent of $i$. This in turn means that $j\le i$ due to the Assumption Well-sorted. A contradiction is reached, so  Claim 2 holds.
\end{description}
Since $G$ is a dependency graph of $\X$, Claims 1 and 2 indicate that $\X_{V(T_i)}$ is independent of $\{\X_{S},\X_{[i-1]\setminus V(T_i)}\}$
. Then
\[
\begin{array}{rl}
&\P(\X_S=\mathbf{y}| X_j=x_j, j\in [i-1]\setminus V(T_i))\\
=&\P(\X_S=\mathbf{y}| X_j=x_j, j\in ([i-1]\setminus V(T_i))\bigcup V(T_i))  \\
=&\P(\X_S=\mathbf{y}| X_j=x_j, j\in [i])\\
=&\P( \X_S=\mathbf{y} | \X_{[i]}=\mathbf{x}_{[i]}).
\end{array}
\]
Likewise, we also have 
\[
\P(\X_S= \y | X_j=x'_j, j\in [i-1]\setminus V(T_i))\\
=\P( \X_S= \y  | \X_{[i]}=\x '_{[i]})
\]
Since $\x $ and $\x '$ differ only in the $i$-th entry,
\[
\P(\X_S= \y | X_j=x_j, j\in [i-1]\setminus V(T_i))
=\P(\X_S= \y | X_j=x'_j, j\in [i-1]\setminus V(T_i)).
\]
As a result, $\P( \X_S= \y  | \X_{[i]}=\x _{[i]})
= \P( \X_S= \y  | \X_{[i]}=\x '_{[i]})$,
this completes the proof of Lemma~\ref{lm:indepofxi}.
\end{proof}
Then we construct the jointly-distributed random vectors $(\mathbf{Y},\mathbf{Z})\in \mathbf{\Omega}^2$ with respect to the fixed $i,\mathbf{x}$, and $\mathbf{x}'$. Specifically, $\mathbf{Y}=(Y_1,\ldots,Y_n)$ and $\mathbf{Z}=(Z_1,\ldots,Z_n)$ are defined as below.
\begin{enumerate}
\item $\Y_{[i]}=\x _{[i]}$
\item For any vector $ \y \in \mathbf{\Omega}_{[i+1,n]}$,
\[
\P(\Y_{[i+1,n]}= \y )=\P(\X_{[i+1,n]}= \y |\X_{[i]}=\x _{[i]})
\]
\item $\mathbf{Z}_{[i]}=\mathbf{x'}_{[i]}, \mathbf{Z}_S=\Y_S$.
\item For any vector $\mathbf{z}\in \mathbf{\Omega}_S$ and element $z\in \Omega_{p(i)}$,
\[
\P(Z_{p(i)}=z|\mathbf{Z}_S=\mathbf{z})
=\P(X_{p(i)}=z|\X_{[i]}=\x '_{[i]},\X_S=\mathbf{z})
\]
\end{enumerate}
The next lemma states that $(\mathbf{Y},\mathbf{Z})$ has the desired marginal distribution.
\begin{lm}\label{coupling}
For any vector $\mathbf{y}\in\mathbf{\Omega}_{[i+1,n]}$, we have 
\begin{enumerate}
\item $\P(\mathbf{Y}_{[i+1,n]}=\mathbf{y})=\P(\X_{[i+1,n]}=\mathbf{y}|\X_{[i]}=\mathbf{x}_{[i]})$, \label{marginalofY}
\item $\P(\mathbf{Z}_{[i+1,n]}=\mathbf{y})=\P(\X_{[i+1,n]}=\mathbf{y}|\X_{[i]}=\mathbf{x}'_{[i]})$. \label{marginalofZ}
\end{enumerate}
\end{lm}
\begin{proof}
(\ref{marginalofY}) holds by the definition of $\mathbf{Y}$.
To prove (\ref{marginalofZ}), arbitrarily choose $\mathbf{y}=(y_{i+1},\ldots,y_n)\in\mathbf{\Omega}_{[i+1,n]}$.  Then we have
\[
\begin{array}{rl}
\P(\mathbf{Z}_{[i+1,n]}=\mathbf{y})
=&\P(\mathbf{Z}_S=\mathbf{y}_S)\P(Z_{p(i)}=y_{p(i)}|\mathbf{Z}_S=\mathbf{y}_S)\\
=&\P(\mathbf{Y}_S=\mathbf{y}_S)\P(Z_{p(i)}=y_{p(i)}|\mathbf{Z}_S=\mathbf{y}_S)\\
=& \P(\X_S=\mathbf{y}_S|\X_{[i]}=\mathbf{x}_{[i]})
\P(X_{p(i)}=y_{p(i)}|\X_{[i]}=\mathbf{x}'_{[i]},\X_S=\mathbf{y}_S)\\
=& \P(\X_S=\mathbf{y}_S|\X_{[i]}=\mathbf{x}'_{[i]})
\P(X_{p(i)}=y_{p(i)}|\X_{[i]}=\mathbf{x}'_{[i]},\X_S=\mathbf{y}_S)\\
=& \P(\X_{[i+1,n]}=\mathbf{y}|\X_{[i]}=\mathbf{x}'_{[i]}).
\end{array}
\] 
where the fourth equality is due to Lemma~\ref{lm:indepofxi}.
\end{proof}

\begin{lm}
$\E[ f(\X)| \X_{[i]}=\mathbf{x}_{[i]}] - \E[ f(\X)| \X_{[i]}=\mathbf{x}'_{[i]} ]  
\le c_i + c_{p(i)}$.
\label{martingale differences lemma}
\end{lm}
\begin{proof}
By the definition of random vectors $\mathbf{Y},\mathbf{Z}$ and Lemma \ref{coupling}, 
\[
\begin{array}{rl}
\E[ f(\X)| \X_{[i]}=\mathbf{x}_{[i]}] - \E[ f(\X)| \X_{[i]}=\mathbf{x}'_{[i]} ]
=& \E[ f(\mathbf{Y}) ] - \E[ f(\mathbf{Z}) ] \\
=& \E[ f(\mathbf{Y})-f(\mathbf{Z})] \\
\le & \E\left[ \sum^n_{j=1}c_j \mathbf{1}_{Y_j\ne Z_j} \right] \\
\le&c_i+c_{p(i)}.
\end{array}
\]
the first equality is due to the coupling constructed before, 
and the first inequality is by triangle inequality and $\c$-Lipschitz properties of $f$.
\end{proof}
We are now ready to prove Theorem~\ref{Concentration Inequality for Dependency Tree}.
\begin{proof}[Proof of Theorem~\ref{Concentration Inequality for Dependency Tree}]
By Lemma~\ref{McDiarmid's Lemma} and Lemma~\ref{martingale differences lemma}
\[
\begin{array}{rl}
\P(f(\X) - \E  [f(\X)] \ge t) 
\le& \exp \left( -\dfrac{2t^2}
{ \sum_{j \in V(G) \setminus \{n\} }( c_j + c_{p(j)} )^2 + c_n^2} \right) \\
=& \exp \left( -\dfrac{2t^2}
{ \sum_{\langle j,k \rangle \in E(G)} ( c_j+c_k )^2 + c_{\min}^2 } \right)
\end{array}
\]
the last equality is because the root $n$ has no parent
and the \textbf{Well-rooted} assumption.
\end{proof}

\subsection{Proof of Theorem~\ref{Concentration Inequality for Dependency forest}}

\begin{proof}[Proof of Theorem~\ref{Concentration Inequality for Dependency forest}]\label{forest}
The proof is similar to that of Theorem~\ref{Concentration Inequality for Dependency Tree}. Without loss of generality, we assume that each component of the forest $G$ are well-rooted and well-sorted. Then the proofs of Lemma \ref{lm:indepofxi}-\ref{martingale differences lemma} remain valid, since variables in different components are independent. As a result, the theorem holds due to Lemma \ref{McDiarmid's Lemma}.

% Since $G$ is a forest consisting of trees $\{ T_i \}_{i \in [k]}$,
% and trees are disconnected with each other, thus, we only need to establish results for each tree.
% This is because from Lemma~\ref{McDiarmid's Lemma}, we can see that the exposure 
% of a vertex in a tree will have no influence on other trees. 
% Then, for each tree, we assume the vertices are \textbf{Well-rooted} and \textbf{Well-sorted},
% then, the result will be the same as that of Theorem~\ref{Concentration Inequality for Dependency Tree}.
\end{proof}

\subsection{Proof of Theorem~\ref{Concentration Inequality for Dependency Forest 1}}\label{generalgraph}

\begin{lm}
Suppose that $f: \mathbf{\Omega}\rightarrow \R$ is a $\c$-Lipschitz function and $G$ is a dependency graph of a random vector $\X$ that takes values in $\mathbf{\Omega}$.  For any $t>0$ and any $(\phi,F)\in \Phi(G)$ with $F$ consisting of trees $\{ T_i \}_{i \in [k]}$, 
the following inequality holds:
\[
\P( f(\X) - \E[f(\X)] \ge t )
\le \exp \left( - \dfrac{2t^2}{ 
\sum_{\langle u,v \rangle \in E(F)} ( \widetilde{c}_u + \widetilde{c}_v)^2 + \sum_{i=1}^k \widetilde{c}_{\min,i}^2 } \right)
\]
where $\widetilde{c}_u=\sum_{i\in \phi^{-1}(u)} c_i$ and $\widetilde{c}_{\min,i}=\min_{u\in V(T_i)}\widetilde{c}_u$. Here, $\phi^{-1}(u)$ is the set of pre-images of $u$.
\label{Concentration Inequality for Dependency Graph}
\end{lm}
\begin{proof}
For any $u\in V(F)$, define a random vector $\mathbf{Y}_u=\{X_i\}_{i\in \phi^{-1}(u)}$. Treat each $\mathbf{Y}_u$ as a random variable. Define a new random vector $\mathbf{Y}=(\mathbf{Y}_u)_{u\in V(F)}$, and let $g(\mathbf{Y})=f(\X)$. It is easy to check that $g$ is $\widetilde{\c}$-Lipschitz, where $\widetilde{\c}=(\widetilde{c}_u)_{u\in V(F)}$. The theorem immediately follows from Theorem~\ref{Concentration Inequality for Dependency forest}.
\end{proof}

Lemma~\ref{Concentration Inequality for Dependency Graph} 
immediately implies Theorem~\ref{Concentration Inequality for Dependency Forest 1}
by the definition of forest complexity.

\section{Omitted Proofs in Section 4}
\label{Omitted Proofs in Section 4}

\subsection{Proof of Lemma~\ref{stability concentration}}\label{Phiisconcentrated}

The following technical lemma is needed.
\begin{lm}[\cite{bousquet2002stability}]
Given a $\beta_n$-uniformly stable learning algorithm $\A$,
for any $\S, \S' \in (\mathcal X \times \mathcal Y)^n$ that differ only in one entry, 
it holds that
\[
| \Phi_{\A}(\S) - \Phi_{\A}(\S')| \le 4\beta_n + \frac{M}{n}
\]
\label{generalization error bound bounded difference}
\end{lm}

\begin{proof}
In the literature, Lemma~\ref{generalization error bound bounded difference}
was proved for i.i.d. data, actually,
the proof remains valid in our setting.
Assume $\S$, $\S'$ differ only in $i$-th entry,
and denote $\S'$ as $\S^i$
\[
\S^i 
= ( (x_1, y_1), \ldots, (x_{i - 1}, y_{i - 1}), 
(x_i', y_i'), (x_{i + 1}, y_{i + 1}) \ldots, (x_m, y_m))
\]
and the marginal distribution of $(x_i', y_i')$ is also $D$.

Notice that we do not require the data to be i.i.d., 
samples are dependent with the same marginal probability distribution $D$.
First, we bound $R (f^{\A}_{\S}) - R (f^{\A}_{\S^i})$
\begin{align}
& | R (f^{\A}_{\S}) - R (f^{\A}_{\S^i}) | \\
\le& | R (f^{\A}_{\S}) - R ( f^{\A}_{\S^{\setminus i}} ) | 
+ | R(f^{\A}_{\S^{\setminus i}}) - R (f^{\A}_{\S^i}) | 
\\
=& | \E_D [\ell(y, f^{\A}_{\S}(x))] - \E_D [\ell(y, f^{\A}_{\S^{\setminus i}}(x))] | 
+ | \E_D [\ell(y, f^{\A}_{\S^{\setminus i}}(x))] - \E_D [ \ell(y, f^{\A}_{\S^i}(x))] |
\\
=& | \E_D [\ell(y, f^{\A}_{\S}(x)) - \ell(y, f^{\A}_{\S^{\setminus i}}(x))] | 
+ | \E_D [\ell(y, f^{\A}_{\S^{\setminus i}}(x)) - \ell(y, f^{\A}_{\S^i}(x))] | 
\\
\le& 2\beta_n
\end{align}
then, we bound $\widehat R (f^{\A}_{\S}) - \widehat R_{\S^i} (f^{\A}_{\S^i}) $
\begin{align}
& n | \widehat R (f^{\A}_{\S}) - \widehat R_{\S^i} (f^{\A}_{\S^i}) | \\
=& \left| \sum_{ (x_j, y_j) \in \S } \ell(y_j, f^{\A}_{\S}(x_j))
- \sum_{ (x_j, y_j) \in \S^i } \ell(y_j, f^{\A}_{\S^i}(x_j)) \right| 
\\
\le& \sum_{j\ne i} |\ell(y_j, f^{\A}_{\S}(x_j)) - \ell(y_j, f^{\A}_{\S^i}(x_j)) | 
+ | \ell(y_i, f^{\A}_{\S}(x_i)) - \ell(y_i^\prime, f^{\A}_{\S^i}(x_i^\prime)) | 
\\
\le& \sum_{j\ne i} |\ell(y_j, f^{\A}_{\S}(x_j)) - \ell(y_j, f^{\A}_{\S^{\setminus i}}(x_j)) | 
+ \sum_{j\ne i} | \ell(y_j, f^{\A}_{\S^{\setminus i}}(x_j)) - \ell(y_j, f^{\A}_{\S^i}(x_j)) | 
\nonumber \\ 
&+ | \ell(y_i, f^{\A}_{\S}(x_i)) - \ell(y_i^\prime, f^{\A}_{\S^i}(x_i^\prime)) | 
\\
\le& 2n\beta_n + M
\end{align}
combining above bounds, we have
\[
\begin{array}{rl}
| \Phi_{\A}(\S) - \Phi_{\A}(\S^i)|
=& | ( R (f^{\A}_{\S}) - \widehat R (f^{\A}_{\S}) ) - ( R (f^{\A}_{\S^i}) - \widehat R_{\S^i} (f^{\A}_{\S^i}) ) | \\
\le& | R (f^{\A}_{\S}) - R (f^{\A}_{\S^i}) | + | \widehat R (f^{\A}_{\S}) - \widehat R_{\S^i} (f^{\A}_{\S^i}) | \\
\le& 4\beta_n + \frac{M}{n}
\end{array}
\]
\end{proof}
combining Lemma~\ref{generalization error bound bounded difference}
and Theorem~\ref{Concentration Inequality for Dependency Forest 1}
leads to Lemma~\ref{stability concentration}.

\subsection{Proof of Lemma~\ref{stability expectation lemma via degree}} \label{Phihassmallexpectation}

We introduce a technical lemma before the proof of Lemma~\ref{stability expectation lemma via degree}.
\begin{lm}
Given a sample $\S$ of size $n$ with dependency graph $G$,
assume that the learning algorithm $\A$ is $\beta_i$-uniformly stable for any $i\le n$.
Suppose the maximum degree of $G$ is $\Delta$. Let $\beta_{n, \Delta} = \max_{i \in [0,\Delta]} \beta_{n-i}$. It holds that
\[
\max_{(x_i, y_i) \in \S} \E_{\S, (x, y)} [\ell(y, f^{\A}_{\S}(x)) - \ell(y_i, f^{\A}_{\S}(x_i))]
\le 2 \beta_{n, \Delta}(\Delta+1).
\]
\end{lm}
\begin{proof}
For any $i\in [n]$, suppose $N_G^+(i)=\{j_1,\ldots,j_{n_i}\}$ with $j_{k-1}>j_k$. 
Define $\S^{(i,0)}=\S$ and for $k\in [n_i]$, $\S^{(i,k)}$ is obtained from $\S^{(i,k-1)}$ by removing the $j_k$-th entry. 
By uniform stability of $\A$, for any $(x,y)\in \mathcal X \times\mathcal Y$ and $k\in [n_i]$, 
\[
|\ell(y, f^{\A}_{\S^{(i,k-1)}}(x))-\ell(y, f^{\A}_{\S^{(i,k)}}(x))|\le \beta_{n, \Delta}
\]
we have the decomposition via telescoping
\[
\ell(y, f^{\A}_{\S}(x)) 
= \sum_{k=1}^{n_i} (\ell(y, f^{\A}_{\S^{(i,k-1)}}(x))-\ell(y, f^{\A}_{\S^{(i,k)}}(x)) 
+ \ell(y, f^{\A}_{\S^{(i,n_i)}}(x))
\]
similarly
\[
\ell(y_i, f^{\A}_{\S}(x_i)) 
= \sum_{k=1}^{n_i} (\ell(y_i, f^{\A}_{\S^{(i,k-1)}}(x_i))-\ell(y_i, f^{\A}_{\S^{(i,k)}}(x_i)) 
+ \ell(y_i, f^{\A}_{\S^{(i,n_i)}}(x_i))
\]
Thus, we have
\[
\begin{array}{rl}
& \ell(y, f^{\A}_{\S}(x)) - \ell(y_i, f^{\A}_{\S}(x_i))  \\
=& \sum_{k=1}^{n_i} \left( (\ell(y, f^{\A}_{\S^{(i,k-1)}}(x))-\ell(y, f^{\A}_{\S^{(i,k)}}(x)))
- (\ell(y_i, f^{\A}_{\S^{(i,k)}}(x_i))-\ell(y_i, f^{\A}_{\S^{(i,k-1)}}(x_i))) \right)  \\
&+ \ell(y, f^{\A}_{\S^{(i,n_i)}}(x))-\ell(y_i, f^{\A}_{\S^{(i,n_i)}}(x_i)) \\
\le& \sum_{k=1}^{n_i} | \ell(y, f^{\A}_{\S^{(i,k-1)}}(x))-\ell(y, f^{\A}_{\S^{(i,k)}}(x)) | \\
&+ \sum_{k=1}^{n_i} | \ell(y_i, f^{\A}_{\S^{(i,k)}}(x_i))-\ell(y_i, f^{\A}_{\S^{(i,k-1)}}(x_i)) |
+ \ell(y, f^{\A}_{\S^{(i,n_i)}}(x))-\ell(y_i, f^{\A}_{\S^{(i,n_i)}}(x_i))\\
\le& 2 n_i\beta_{n, \Delta} + \ell(y, f^{\A}_{\S^{(i,n_i)}}(x))-\ell(y_i, f^{\A}_{\S^{(i,n_i)}}(x_i))
\end{array}
\] 
As a result,
\[
\begin{array}{rl}
&\E_{\S, (x, y)} [\ell(y, f^{\A}_{\S}(x)) - \ell(y_i, f^{\A}_{\S}(x_i))] \\
=& \E_{\S, (x, y)} [ \ell(y, f^{\A}_{\S^{(i,n_i)}}(x)) - \ell(y_i, f^{\A}_{\S^{(i,n_i)}}(x_i)) ] 
+ 2 n_i \beta_{n, \Delta} \\
\le& \E_{\S, (x, y)} [ \ell(y, f^{\A}_{\S^{(i,n_i)}}(x)) - \ell(y_i, f^{\A}_{\S^{(i,n_i)}}(x_i)) ] 
+ 2 \beta_{n, \Delta}(\Delta+1) \\
=& \E_{\S, (x, y)} [ \ell(y, f^{\A}_{\S^{(i,n_i)}}(x)) ] 
- \E_\S [\ell(y_i, f^{\A}_{\S^{(i,n_i)}}(x_i)) ]
+ 2 \beta_{n, \Delta}(\Delta+1) \\
=& \E_{\S^{(i,n_i)}, (x, y)} [ \ell(y, f^{\A}_{\S^{(i,n_i)}}(x)) ]
- \E_{\S^{(i,n_i)},(x_i,y_i)} [\ell(y_i, f^{\A}_{\S^{(i,n_i)}}(x_i)) ] + 2 \beta_{n, \Delta}(\Delta+1) \\
=& 2 \beta_{n, \Delta}(\Delta+1)
\end{array}
\]
The last equality is because
$(x_i, y_i)$ and $(x, y)$ are independent of $\S^{(i,n_i)}$ and have the same distribution.
\end{proof}

\begin{proof}[Proof of Lemma~\ref{stability expectation lemma via degree}]
\[
\begin{array}{rl}
\E_{\S} [ \Phi_{\A}(\S) ]
=& \E_{\S} 
[ \E_{(x, y)} [\ell(y, f^{\A}_{\S}(x))] - \frac{1}{n} \sum_{i = 1}^n \ell(y_i, f^{\A}_{\S}(x_i)) ] \\
=& \frac{1}{n} \sum_{i = 1}^n \E_{\S, (x, y)}
[ \ell(y, f^{\A}_{\S}(x)) - \ell(y_i, f^{\A}_{\S}(x_i)) ] \\
\le& 2 \beta_{n, \Delta}(\Delta+1)
\end{array}
\]	
\end{proof}

\end{document}